\documentclass[review]{elsarticle}

\usepackage{lineno,hyperref}
\usepackage[english]{babel}
\usepackage[utf8]{inputenc}
\usepackage{microtype}

\usepackage{amsmath}
\usepackage{amsfonts}
\usepackage{amsthm}
\usepackage{booktabs}
\usepackage{todonotes}

\modulolinenumbers[5]

\journal{Neurocomputing}

\newtheorem{prop}{Proposition}
\newtheorem{lemma}{Lemma}
\newtheorem{definition}{Definition}
\newtheorem{theorem}{Theorem}

\begin{document}

\begin{frontmatter}

\title{Mean Absolute Percentage Error for Regression Models}

\author[viadeo,samm]{Arnaud de Myttenaere}
\ead{ademyttenaere@viadeoteam.com}

\author[viadeo]{Boris Golden}
\ead{bgolden@viadeoteam.com}

\author[cri]{Bénédicte Le Grand}
\ead{Benedicte.Le-Grand@univ-paris1.fr}

\author[samm]{Fabrice Rossi\corref{mycorrespondingauthor}}
\cortext[mycorrespondingauthor]{Corresponding author}
\ead{Fabrice.Rossi@univ-paris1.fr}
\address[viadeo]{Viadeo, 30 rue de la Victoire, 75009 Paris - France}
\address[cri]{Centre de Recherche en Informatique -- Universit\'e Paris 1 Panth\'eon - Sorbonne, 
90 rue de Tolbiac, 75013 Paris - France}
\address[samm]{SAMM EA 4534 -- Universit\'e Paris 1 Panth\'eon - Sorbonne, 
90 rue de Tolbiac, 75013 Paris - France}

\begin{abstract}
  We study in this paper the consequences of using the Mean Absolute
  Percentage Error (MAPE) as a measure of quality for regression models. We
  prove the existence of an optimal MAPE model and we show the universal
  consistency of Empirical Risk Minimization based on the MAPE. We also show
  that finding the best model under the MAPE is equivalent to doing weighted
  Mean Absolute Error (MAE) regression, and we apply this weighting strategy
  to kernel regression. The behavior of the MAPE kernel regression is
  illustrated on simulated data.
\end{abstract}

\begin{keyword}
Mean Absolute Percentage Error; Empirical Risk Minimization; Consistency; Optimization; Kernel Regression.
\end{keyword}

\end{frontmatter}

\linenumbers

\section{Introduction}
Classical regression models are obtained by choosing a model that minimizes an
empirical estimation of the Mean Square Error (MSE). Other quality measures
are used, in general for robustness reasons. This is the case of the Huber
loss \cite{Huber1964} and of the Mean Absolute Error (MAE, also know as median
regression), for instance. Another example of regression quality measure is
given by the Mean Absolute Percentage Error (MAPE). If $x$ denotes the vector
of explanatory variables (the input to the regression model), $y$ denotes the
target variable and $g$ is a regression model, the MAPE of $g$ is obtained by
averaging the ratio $\frac{|g(x)-y|}{|y|}$ over the data.

The MAPE is often used in practice because of its very intuitive
interpretation in terms of relative error. The use of the MAPE is relevant in
finance, for instance, as gains and losses are often measured in relative
values. It is also useful to calibrate prices of products, since customers are
sometimes more sensitive to relative variations than to absolute variations.

In real world applications, the MAPE is frequently used when the quantity to
predict is known to remain way above zero. It was used for instance as the
quality measure in a electricity consumption forecasting contest organized by
GdF ecometering on datascience.net\footnote{\url{http//www.datascience.net},
  see \url{https://www.datascience.net/fr/challenge/16/details} for details on
  this contest.}. More generally, it has been argued that the MAPE is very
adapted for forecasting applications, especially in situations where enough
data are available, see e.g. \cite{ArmstrongCollopy1992}.

We study in this paper the consequences of using the MAPE as the quality
measure for regression models. Section \ref{sec:setting} introduces our
notations and the general context. It recalls the definition of the
MAPE. Section \ref{sec:exist-MAPE-regr} is dedicated to a first important
question raised by the use of the MAPE: it is well known that the optimal
regression model with respect to the MSE is given by the regression function
(i.e., the conditional expectation of the target variable knowing the
explanatory variables). Section \ref{sec:exist-MAPE-regr} shows that an
optimal model can also be defined for the MAPE. Section
\ref{sec:effects-MAPE-compl} studies the consequences of replacing MSE/MAE by
the MAPE on capacity measures such as covering numbers and Vapnik-Chervonenkis
dimension. We show in particular that MAE based measures can be used to upper
bound MAPE ones. Section \ref{sec:consistency-MAPE} proves a universal
consistency result for Empirical Risk Minimization applied to the MAPE, using
results from Section \ref{sec:effects-MAPE-compl}. Finally, Section
\ref{sec:MAPE-kern-regr} shows how to perform MAPE regression in practice. It
adapts quantile kernel regression to the MAPE case and studies the behavior of
the obtained model on simulated data. 

\section{General setting and notations}\label{sec:setting}
We use in this paper a standard regression setting in which the data are fully
described by a random pair $Z=(X,Y)$ with values in
$\mathbb{R}^d\times \mathbb{R}$. We are interested in finding a good model for
the pair, that is a (measurable) function $g$ from $\mathbb{R}^d$ to
$\mathbb{R}$ such that $g(X)$ is ``close to'' $Y$. In the classical regression
setting, the closeness of $g(X)$ to $Y$ is measured via the $L_2$ risk, also
called the mean squared error (MSE), defined by
\begin{equation}\label{eq:ltwo:risk}
L_2(g)=L_{MSE}(g)=\mathbb{E}(g(X)-Y)^2.
\end{equation}
In this definition, the expectation is computed by respect to the random pair
$(X,Y)$ and might be denoted $\mathbb{E}_{X, Y}(g(X)-Y)^2$ to make this point
explicit. To maintain readability, this explicit notation will be used only in
ambiguous settings. 

Let $m$ denote the regression function of the problem, that is the function
from $\mathbb{R}^d$ to $\mathbb{R}$  given by 
\begin{equation}
  \label{eq:regression:function}
m(x)=\mathbb{E}(Y|X=x).  
\end{equation}
It is well known (see e.g. \cite{gyorfi_etal_DFTNR2002}) that the regression
function is the best model in the case of the mean squared error in the sense
that $L_2(m)$ minimizes $L_2(g)$ over the set of all measurable functions from
$\mathbb{R}^d$ to $\mathbb{R}$. 

More generally, the quality of a model is measured via a \textbf{loss
  function}, $l$, from $\mathbb{R}^2$ to $\mathbb{R}^+$. The point-wise loss
of the model $g$ is $l(g(X),Y)$ and the \textbf{risk} of the model is
\begin{equation}
  \label{eq:General:Risk}
L_l(g)=  \mathbb{E}(l(g(X),Y)).
\end{equation}
For example, the squared loss, $l_2=l_{MSE}$ is defined as $l_2(p,y)=(p-y)^2$. It
leads to the $L_{MSE}$ risk defined above as $L_{l_2}(g)=L_{MSE}(g)$.

The \textbf{optimal risk} is the infimum of $L_l$ over measurable functions,
that is
\begin{equation}
  \label{eq:Optimal:Risk}
L^*_l=\inf_{g\in \mathcal{M}(\mathbb{R}^d,\mathbb{R})}L_l(g),
\end{equation}
where $\mathcal{M}(\mathbb{R}^d,\mathbb{R})$ denotes the set of measurable
functions from $\mathbb{R}^d$ to $\mathbb{R}$. As recalled above we have
\[
L^*_{MSE}=L^*_2=L^*_{l_2}=\mathbb{E}_{X, Y}(m(X)-Y)^2=\mathbb{E}_{X, Y}\left\{(\mathbb{E}(Y|X)-Y\right)^2\}
\]

As explained in the introduction, there are practical situations in which the
$L_2$ risk is not a good way of measuring the closeness of $g(X)$ to $Y$. We
focus in this paper on the case of the mean absolute percentage error (MAPE)
as an alternative to the MSE. Let us recall that the loss function associated to
the MAPE is given by
\begin{equation}
  l_{MAPE}(p,y)=\frac{|p-y|}{|y|},
\end{equation}
with the conventions that for all $a\neq 0$, $\frac{a}{0}=\infty$ and that
$\frac{0}{0}=1$. Then the MAPE-risk of model $g$ is
\begin{equation}\label{eq:lMAPE:risk}
L_{MAPE}(g)=L_{l_{MAPE}}(g)=\mathbb{E}\left(\frac{|g(X)-Y|}{|Y|}\right).
\end{equation}
Notice that according to Fubini's theorem, $L_{MAPE}(g)<\infty$ implies in
particular that $\mathbb{E}(|g(X)|)<\infty$ and thus that interesting models
belong to $L^1(\mathbb{P}_X)$, where $\mathbb{P}_X$ is the probability measure
on $\mathbb{R}^d$ induced by $X$. 

We will also use in this paper the mean absolute error (MAE). It is based on
the absolute error loss, $l_{MAE}=l_1$ defined by
$l_{MAE}(p,y)=|p-y|$. As other risks, the MAE-risk is given by
\begin{equation}
  \label{eq:lmae:risk}
L_{MAE}(g)=L_{l_{MAE}}(g)=\mathbb{E}(|g(X)-Y|).  
\end{equation}

\section{Existence of the MAPE-regression function}\label{sec:exist-MAPE-regr}
A natural theoretical question associated to the MAPE is whether an optimal
model exists. More precisely, is there a function $m_{MAPE}$ such
that for all models $g$, $L_{MAPE}(g)\geq L_{MAPE}(m_{MAPE})$?

Obviously, we have
\[
L_{MAPE}(g)=\mathbb{E}_{X, Y}\left\{\mathbb{E}\left(\frac{|g(X)-Y|}{|Y|}\Bigg|X\right)\right\}.
\]
A natural strategy to study the existence of $m_{MAPE}$ is therefore to
consider a point-wise approximation, i.e. to minimize the conditional
expectation introduced above for each value of $x$.  In other words, we want
to solve, if possible, the optimization problem
\begin{equation}
  \label{eq:MAPE:regression:problem}
m_{MAPE}(x)=\arg\min_{m\in \mathbb{R}}  \mathbb{E}\left(\frac{|m-Y|}{|Y|}\Bigg|X=x\right),
\end{equation}
for all values of $x$. 

We show in the rest of this Section that this problem can be solved. We first
introduce necessary and sufficient conditions for the problem to involve
finite values, then we show that under those conditions, it has at least one
global solution for each $x$ and finally we introduce a simple rule to select
one of the solutions. 

\subsection{Finite values for the point-wise problem}
To simplify the analysis, let us introduce a real valued random variable $T$ and
study the optimization problem
\begin{equation}
  \label{eq:MAPE:regression:problem:nocond}
\min_{m\in \mathbb{R}}  \mathbb{E}\left(\frac{|m-T|}{|T|}\right).
\end{equation}
Depending on the distribution of $T$ and of the value of $m$,
$J(m)=\mathbb{E}\left(\frac{|m-T|}{|T|}\right)$ is not always a finite value,
excepted for $m=0$. In this latter case, for any random variable $T$, $J(0)=1$ using the
above convention.

Let us consider an example demonstrating problems that might arise for
$m\neq 0$. Let $T$ be distributed according to the uniform distribution on
$[-1,1]$. Then
\[
J(m)=\frac{1}{2}\int_{-1}^1\frac{|m-t|}{|t|}dt.
\]
If $m\in]0,1]$, we have
\begin{align*}
J(m)&=
      \frac{1}{2}\int_{-1}^0\left(1-\frac{m}{t}\right)dt+\frac{1}{2}\int_{0}^m\left(\frac{m}{t}-1\right)dt+\frac{1}{2}\int_{m}^1\left(1-\frac{m}{t}\right)dt,\\
&=\underbrace{1-m-\frac{m}{2}\int_{m}^1\frac{1}{t}dt}_{\text{finite
  part}}+\frac{m}{2}\underbrace{\left(\int_{0}^m\frac{1}{t}dt-\int_{-1}^0\frac{1}{t}dt\right)}_{+\infty},\\
&=+\infty.
\end{align*}
This example shows that when $T$ is likely to take values close to 0, then
$J(m)=\infty$ whenever $m\neq 0$. Intuitively, the only situation that leads
to finite values is when  $\frac{1}{|T|}$ as a finite expectation,
that is when the probability that $|T|$ is smaller than $\epsilon$ decreases
sufficiently quickly when $\epsilon$ goes to zero. 

 More formally, we have the following
proposition.
\begin{prop}\label{Proposition:J:finite}
$J(m)<\infty$ for all $m$ if and only if
\begin{enumerate}
\item $\mathbb{P}(T=0)=0$,
\item and
  \begin{align}
    \label{eq:condition:MAPE}
\sum_{k=1}^\infty k\mathbb{P}\left(T\in\left]\frac{1}{k+1},\frac{1}{k}\right]\right)&<\infty,&
\sum_{k=1}^\infty k\mathbb{P}\left(T\in\left[-\frac{1}{k},-\frac{1}{k+1}\right[\right)<\infty. 
  \end{align}
\end{enumerate}
If any of those conditions is not fulfilled, then $J(m)=\infty$ for all $m\neq
0$.
\end{prop}
\begin{proof}
We have
\[
J(m)=\mathbb{E}\left(\mathbb{I}_{T=0}\frac{|m-T|}{|T|}\right)+\mathbb{E}\left(\mathbb{I}_{T>0}\frac{|m-T|}{|T|}\right)
+\mathbb{E}\left(\mathbb{I}_{T<0}\frac{|m-T|}{|T|}\right).
\]
If $\mathbb{P}(T=0)>0$ then for all $m\neq 0$, $J(m)=\infty$. Let us therefore consider
the case $\mathbb{P}(T=0)=0$. We assume $m>0$, the case $m<0$ is completely
identical. We have
\begin{align*}
 J(m)&=\mathbb{E}\left(\mathbb{I}_{T>0}\frac{|m-T|}{|T|}\right)
+\mathbb{E}\left(\mathbb{I}_{T<0}\frac{|m-T|}{|T|}\right),\\
&=\mathbb{P}(T<0)+\mathbb{P}(T> m)-\mathbb{P}(T\in
  ]0,m])+m\mathbb{E}\left(\frac{\mathbb{I}_{T\in ]0,m]}-\mathbb{I}_{T<0}-\mathbb{I}_{T>m}}{T}\right).
\end{align*}
A simple upper bounding gives
\[
0\leq m\mathbb{E}\left(\frac{\mathbb{I}_{T>m}}{T}\right)\leq \mathbb{P}(T> m),
\]
and symmetrically 
\[
0\leq m\mathbb{E}\left(-\frac{\mathbb{I}_{T<-m}}{T}\right)\leq \mathbb{P}(T<-m).
\]
This shows that $J(m)$ is the sum of finite terms and of
$m\mathbb{E}\left(\frac{\mathbb{I}_{T\in ]0,m]}-\mathbb{I}_{T\in
      [-m,0[}}{T}\right)$. Because of the symmetry of the problem, we can
focus on $\mathbb{E}\left(\frac{\mathbb{I}_{T\in ]0,m]}}{T}\right)$. It is
also obvious that $\mathbb{E}\left(\frac{\mathbb{I}_{T\in
      ]0,m]}}{T}\right)$ is finite if and only if $\mathbb{E}\left(\frac{\mathbb{I}_{T\in
      ]0,1]}}{T}\right)$ is finite. 

As pointed out above, this shows that, when $\mathbb{P}(T=0)=0$, $J(m)$ is
finite if and only if both
$\mathbb{E}\left(\frac{\mathbb{I}_{T\in ]0,1]}}{T}\right)$ and
$\mathbb{E}\left(\frac{\mathbb{I}_{T\in [-1,0[}}{T}\right)$ are finite. We
obtain slightly more operational conditions in the rest of the proof. 

Let us therefore introduce the following functions:
\begin{align*}
f^{-}_k(x)&=
  \begin{cases}
    0& \text{if }x\not\in ]\frac{1}{k+1},\frac{1}{k}],\\
    k& \text{if }x\in ]\frac{1}{k+1},\frac{1}{k}],
  \end{cases}
&f^{+}_k(x)&=
  \begin{cases}
    0& \text{if }x\not\in ]\frac{1}{k+1},\frac{1}{k}],\\
    k+1& \text{if }x\in ]\frac{1}{k+1},\frac{1}{k}],
  \end{cases}\\
g^{-}_n&=\sum_{k=1}^nf^{-}_k(x),
&g^{+}_n&=\sum_{k=1}^nf^{+}_k(x),\\
g^{-}&=\sum_{k=1}^\infty f^{-}_k(x),
&g^{+}&=\sum_{k=1}^\infty f^{+}_k(x).
\end{align*}
We have obviously for all $x\in ]0,1]$, $g^{-}(x)\leq \frac{1}{x}\leq
g^{+}(x)$. In addition
\begin{align*}
 \mathbb{E}( g^{+}_n(T))&=\sum_{k=1}^n(k+1)\mathbb{P}\left(T\in\left]\frac{1}{k+1},\frac{1}{k}\right]\right),\\
 \mathbb{E}( g^{-}_n(T))&=\sum_{k=1}^nk\mathbb{P}\left(T\in\left]\frac{1}{k+1},\frac{1}{k}\right]\right)=\mathbb{E}( g^{+}_n(T))-\mathbb{P}\left(T\in\left]\frac{1}{k+1},1\right]\right).
\end{align*}
According to the monotone convergence theorem, 
\begin{equation*}
  \mathbb{E}( g^{+}(T))=\lim_{n\rightarrow \infty}\mathbb{E}( g^{+}_n(T)).
\end{equation*}
The link between  $\mathbb{E}( g^{-}_n(T))$ and $\mathbb{E}( g^{+}_n(T))$
shows that either both $\mathbb{E}( g^{+}(T))$ and $\mathbb{E}( g^{-}(T))$ are
finite, or both are infinite. In addition, we have 
\begin{equation*}
  \mathbb{E}( g^{-}(T))\leq \mathbb{E}\left(\frac{\mathbb{I}_{T\in
      ]0,1]}}{T}\right)\leq  \mathbb{E}( g^{+}(T)), 
\end{equation*}
therefore $\mathbb{E}\left(\frac{\mathbb{I}_{T\in ]0,1]}}{T}\right)$ is finite
if and only if $\mathbb{E}( g^{-}(T))$ is finite. So a sufficient and
necessary condition for $\mathbb{E}\left(\frac{\mathbb{I}_{T\in
      ]0,1]}}{T}\right)$ to be finite is
\begin{equation*}
\sum_{k=1}^\infty k\mathbb{P}\left(T\in\left]\frac{1}{k+1},\frac{1}{k}\right]\right)<\infty. 
\end{equation*}
A symmetric derivation shows that $\mathbb{E}\left(-\frac{\mathbb{I}_{T\in ]-1,0]}}{T}\right)$ is finite
if and only if
\begin{equation*}
\sum_{k=1}^\infty k\mathbb{P}\left(T\in\left[-\frac{1}{k},-\frac{1}{k+1}\right[\right)<\infty. 
\end{equation*}
\end{proof}
The conditions of Proposition \ref{Proposition:J:finite} can be used to
characterize whether $\mathbb{P}(T\in ]0,\epsilon])$ decreases
  sufficiently quickly to ensure that $J$ is not (almost) identically equal to
  $+\infty$. For instance, if $\mathbb{P}(T\in ]0,\epsilon])=\epsilon$, then 
\[
k\mathbb{P}\left(T\in\left]\frac{1}{k+1},\frac{1}{k}\right]\right)=\frac{1}{k+1},
\]
and the sum diverges, leading to $J(m)=\infty$ (for $m\neq 0$). On the
contrary, if  $\mathbb{P}(T\in
]0,\epsilon])=\epsilon^{2}$, then  
\[
k\mathbb{P}\left(T\in\left]\frac{1}{k+1},\frac{1}{k}\right]\right)=\frac{2k+1}{k(k+1)^2},
\]
and thus the sum converges, leading to $J(m)<\infty$ for all $m$ (provided
similar conditions hold for the negative part of $T$).

\subsection{Existence of a solution for the point-wise problem}
If the conditions of Proposition \ref{Proposition:J:finite} are not fulfilled,
$J(m)$ is infinite excepted in $m=0$ and therefore
$\arg\min_{m\in\mathbb{R}}J(m)=0$. When they are fulfilled, we have to show
that $J(m)$ has at least one global minimum. This is done in the following proposition.
\begin{prop}\label{Proposition:J:asminimum}
Under the conditions of Proposition \ref{Proposition:J:finite}, $J$ is convex
and has at least one global minimum. 
\end{prop}
\begin{proof}
We first note that $J$ is convex. Indeed for all $t\neq 0$, $m\mapsto
\frac{|m-t|}{|t|}$ is obviously convex. Then the linearity of the expectation
allows to conclude (provided $J$ is finite everywhere as guaranteed by the
hypotheses).

As $\mathbb{P}(T=0)=0$, there is $[a,b]$, $a<b$ such that
$\mathbb{P}(T\in[a,b])>0$ with either $a>0$ or $b<0$. Let us assume $a>0$, the
other case being symmetric. Then for $t\in[a,b]$, $\frac{1}{b}\leq
\frac{1}{t}\leq \frac{1}{a}$. If $m>b$, then for $t\in[a,b]$
\begin{equation*}
\frac{|m-t|}{|t|}=\frac{m}{t}-1\geq \frac{m}{b}-1.
\end{equation*}
Then 
\begin{align*}
J(m)&\geq  \mathbb{E}\left(\frac{\mathbb{I}_{T\in [a,b]}|m-T|}{|T|}\right),\\
&\geq \left(\frac{m}{b}-1\right)\mathbb{P}(T\in[a,b]),
\end{align*}
and therefore $\lim_{m\rightarrow+\infty}J(m)=+\infty$.

Similarly, if $m<0<a$, then for $t\in[a,b]$
\begin{equation*}
\frac{|m-t|}{|t|}=1-\frac{m}{t}\geq 1-\frac{m}{b},
\end{equation*}
and then
\begin{equation*}
J(m)\geq   \left(1-\frac{m}{b}\right)\mathbb{P}(T\in[a,b]),
\end{equation*}
and therefore $\lim_{m\rightarrow-\infty}J(m)=+\infty$.

Therefore, $J$ is a coercive function and has at least a local minimum, which
is global by convexity. 
\end{proof}

\subsection{Choosing the minimum}
However, the minimum is not necessary unique, as $J$ is not strictly
convex. In general, the set of global minima will be a bounded interval of
$\mathbb{R}$. In this case, and by convention, we consider the mean value of
the interval as the optimal solution. 

As an example of such behavior, we can consider the case where $T$ is a random
variable on $\{1, 2, 3\}$, such that $\mathbb{P}(T = 1) = 0.3$,
$\mathbb{P}(T = 2) = 0.4$ and $\mathbb{P}(T = 3) = 0.3$. Then the expected
loss is
\[ 
J(m) = 0.3\times \left|m - 1\right| + 0.4 \times \left|\frac{m-2}{2}\right| +
0.3 \times \left|\frac{m-3}{3}\right| 
\]
and the figure~\ref{fig:counterExample} illustrates that there is an infinity
\begin{figure}
\centering
\includegraphics[width =0.7\linewidth]{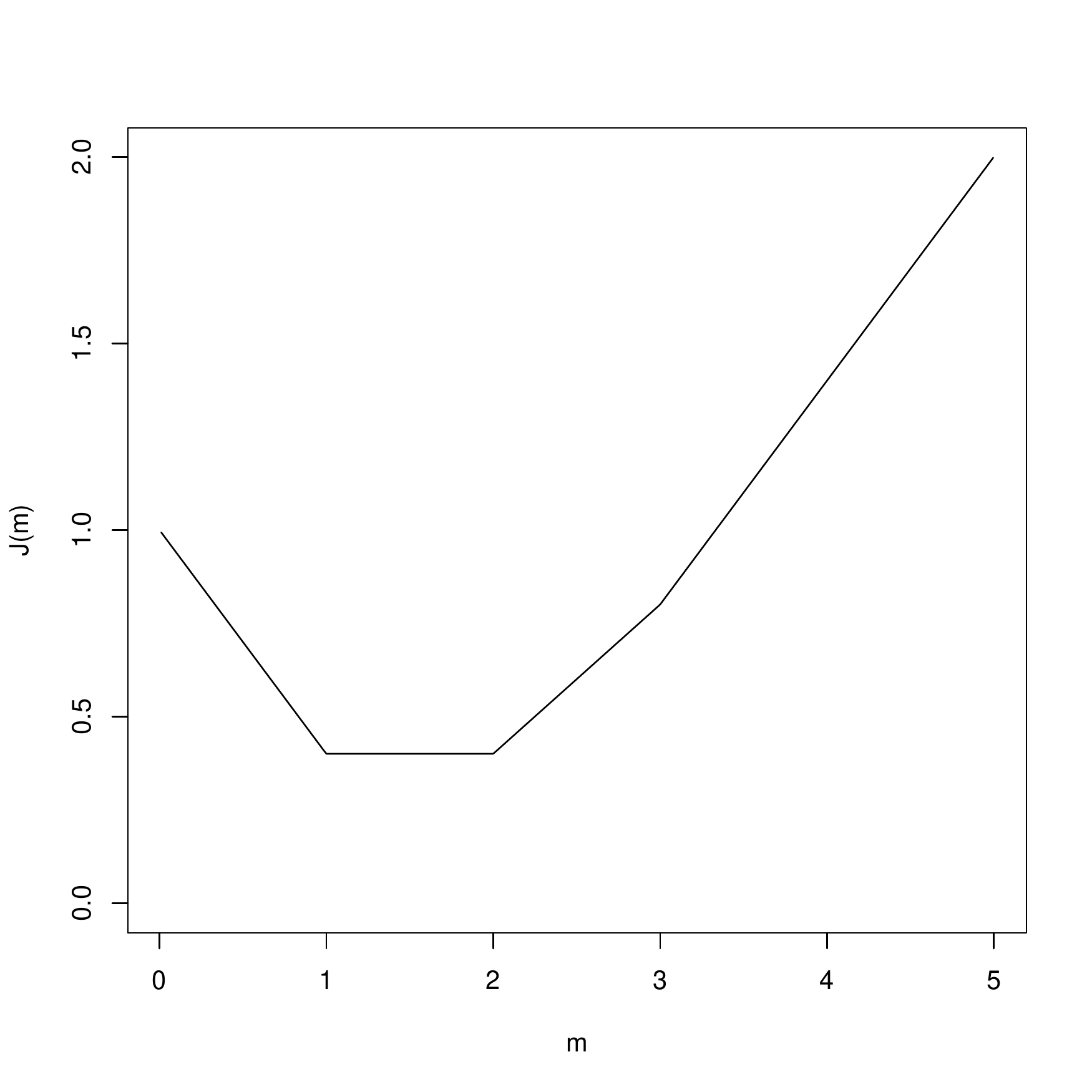}
\caption{Counterexample with an infinite number of solutions.}
\label{fig:counterExample}
\end{figure}
of solutions. Indeed when $m\in[1, 2]$, $J$ becomes
\begin{align*}
J(m)&=0.3\times (m - 1) + 0.4 \times \frac{2-m}{2} +
0.3 \times \frac{3-m}{3},\\
&=(0.3-0.2-0.1)\times m+(-0.3+0.4+0.3),\\
&=0.4.
\end{align*}
Here we define by convention $\arg\min_{m} J(m)=\frac{3}{2}$.

More generally, for any random variable $T$, we have defined a unique value
$m$, which is a global minimum of
$J(m) = \mathbb{E}\left( \left| \frac{m-T}{T}\right|\right)$. Moving back to
our problem, it ensures that the MAPE-regression function $m_{MAPE}$
introduced in \ref{eq:MAPE:regression:problem} is well defined and takes
finite values on $\mathbb{R}^d$. As $m_{MAPE}$ is point-wise optimal, it is
also globally optimal.

\section{Effects of the MAPE on complexity control}\label{sec:effects-MAPE-compl}
One of the most standard learning strategy is the Empirical Risk Minimization
(ERM) principle. We assume given a training set $D_n=(Z_i)_{1\leq i\leq
  N}=(X_i,Y_i)_{1\leq i\leq n}$  which consists in $n$ i.i.d. copies of the
random pair $Z=(X,Y)$. We assume also given a class of models, $G$, which
consists in measurable functions from $\mathbb{R}^d$ to $\mathbb{R}$. Given a
loss function $l$, we denote $L^*_{l,G}=\inf_{g\in G}L_l(g)$. 

The empirical estimate of $L_l(g)$ (called the \textbf{empirical risk}) is
given by
\begin{equation}
  \label{eq:empirical:risk}
\widehat{L}_l(g,D_n)=\frac{1}{n}\sum_{i=1}^nl(g(X_i),Y_i).  
\end{equation}
Then the ERM principle consists in choosing in the class $G$ the model that minimizes
the empirical risk, that is 
\begin{equation}
  \label{eq:ERM:model}
\widehat{g}_{l,D_n,G}=\arg\min_{g\in G}\widehat{L}_l(g,D_n).
\end{equation}
The main theoretical question associated to the ERM principle is how to
control $L_l(\widehat{g}_{l,D_n,G})$ in such a way that it converges to
$L^*_{l,G}$. An extension of this question is whether $L^*_l$ can be reached
if $G$ is allowed to depend on $n$: the ERM is said to be universally strongly
consistent if $L_l(\widehat{g}_{l,D_n,G_n})$ converges to $L^*_l$ almost
surely for any distribution of $(X,Y)$ (see Section
\ref{sec:consistency-MAPE}). 

It is well known (see e.g. \cite{gyorfi_etal_DFTNR2002} chapter 9) that ERM
consistency is related to uniform laws of large numbers (ULLN). In particular, we
need to control quantities of the following form
\begin{equation}
  \label{eq:uniform}
  P\left\{\sup_{g\in G}\left|\widehat{L}_{MAPE}(g,D_n)-L_{MAPE}(g)\right|>\epsilon\right\}.
\end{equation}
This can be done via covering numbers or via the Vapnik-Chervonenkis dimension
(VC-dim) of certain classes of functions derived from $G$. One
might think that general results about arbitrary loss functions can be used to
handle the case of the MAPE. This is not the case as those results generally
assume a uniform Lipschitz property of $l$ (see Lemma 17.6 in
\cite{ab-nnltf-99}, for instance) that is not fulfilled by the MAPE.

The objective of this section is to analyze the effects over covering numbers
(Section \ref{sec:covering-numbers}) and VC-dimension (Section
\ref{sec:vc-dimension}) of using the MAPE as the loss function. It shows also
what type of ULLN results can be obtained based on those analyses (Section
\ref{sec:exampl-unif-laws}). 

\subsection{Classes of functions}
Given a class of models, $G$, and a loss function $l$, we introduce derived
classes, $H(G,l)$ given by
\begin{equation}
H(G,l)=\{h: \mathbb{R}^d\times \mathbb{R}\rightarrow \mathbb{R}^+,\
h(x,y)=l(g(x),y)\ |\ g\in G\},
\end{equation}
and $H^+(G,l)$ given by
\begin{equation}
H^+(G,l)=\{h: \mathbb{R}^d\times \mathbb{R}\times \mathbb{R}\rightarrow \mathbb{R}^+,\
h(x,y,t)=\mathbb{I}_{t\leq l(g(x),y)}\ |\ g\in G\}.
\end{equation}

\subsection{Covering numbers}\label{sec:covering-numbers}
\subsubsection{Notations and definitions}
Let $F$ be a class of positive functions from an arbitrary set $\mathcal{Z}$
to $\mathbb{R}^+$. The supremum norm on $F$ is given by 
\[
\|f\|_{\infty}=\sup_{z\in\mathcal{Z}}|f(z)|.
\]
We also define $\|F\|_{\infty}=\sup_{f\in F}\|f\|_{\infty}$. We have obviously
\[
\forall f\in F,\forall z\in\mathcal{Z}, |f(z)|\leq \|F\|_{\infty}. 
\]
Those definitions will also be used for classes of functions with values in
$\mathbb{R}$ (not only in  $\mathbb{R}^+$), hence the absolute value. 

Let $\kappa$ be a dissimilarity on $F$, that is a positive and symmetric
function from $F^2$ to $\mathbb{R}^+$ that measures how two functions from $F$
are dissimilar (in particular $\kappa(f,f)=0$). Then $\kappa$ can be used to
characterize the complexity of $F$ by computing the $\kappa$
$\epsilon$-covering number of $F$.

\begin{definition}
Let $F$ be a class of positive functions from $\mathcal{Z}$ to $\mathbb{R}^+$ and $\kappa$ a dissimilarity on $F$. For $\epsilon>0$ and $p$ a positive integer, a size $p$ $\epsilon$-cover of $F$ with respect to $\kappa$
is a finite collection $f_1,\ldots,f_p$ of elements of $F$ such that for all
$f\in F$
\[
\min_{1\leq i\leq p}\kappa(f,f_i)<\epsilon.
\]
\end{definition}

Then the $\kappa$ $\epsilon$-covering number of $F$ is defined as follow.

\begin{definition}
  Let $F$ be a class of positive functions from $\mathcal{Z}$ to
  $\mathbb{R}^+$, $\kappa$ be a dissimilarity on $F$ and $\epsilon>0$. Then
  the $\kappa$ $\epsilon$-covering number of $F$,
  $\mathcal{N}(\epsilon,F,\kappa)$, is the size of the smallest $\kappa$
  $\epsilon$-cover of $F$. If such a cover does not exists, the covering
  number is $\infty$.
\end{definition}

The behavior of $\mathcal{N}(\epsilon,F,\kappa)$ with respect to $\epsilon$
characterizes the complexity of $F$ as seen through $\kappa$. If the growth
when $\epsilon\rightarrow 0$ is slow enough (for an adapted choice of
$\kappa$), then some uniform law of large numbers applies (see Lemma
\ref{lemma:ulln:supremum}).

\subsubsection{Supremum covering numbers}\label{sec:supr-cover-numb} 
Supremum covering numbers are based on the supremum norm, that is
\[
\|f_1-f_2\|_{\infty}=\sup_{z\in\mathcal{Z}}|f_1(z)-f_2(z)|.
\]
For classical loss functions, the supremum norm is generally ill-defined on
$H(G,l)$. For instance let $h_1$ and $h_2$ be two functions from $H(G,l_2)$,
generated by $g_1$ and $g_2$ (that is $h_i(x,y)=(g_i(x)-y)^2$). Then
\begin{align*}
|h_1(x,y)-h_2(x,y)|&=|(g_1(x)-y)^2-(g_2(x)-y)^2|\\
&=|g_1(x)^2-g_2(x)^2+2y(g_2(x)-g_1(x))|.
\end{align*}
If $G$ is not reduced to a single function, then there are two functions $g_1$
and $g_2$ and a value of $x$ such that $g_1(x)\neq g_2(x)$. Then
$\sup_{y}|h_1(x,y)-h_2(x,y)|=\infty$. 

A similar situation arises for the MAPE. Indeed, let $h_1$ and $h_2$ be two functions from
$H_{MAPE}$, generated by $g_1$ and $g_2$ in $G$ (that is $h_i(x,y)=\frac{|g_i(x)-y|}{|y|}$). Then
\[
\|h_1-h_2\|_{\infty}=\sup_{(x,y)\in \mathbb{R}^d\times \mathbb{R}}\frac{||g_1(x)-y|-|g_2(x)-y||}{|y|}. 
\]
Thus unless $G$ is very restricted there is always $x$, $g_1$ and $g_2$ such that $g_1(x)\neq 0$ and $|g_2(x)|\neq |g_1(x)|$. Then
for $y>0$, $||g_1(x)-y|-|g_2(x)-y||$ has the general form
$\alpha+\beta y$ with $\alpha>0$ and thus $\lim_{y\rightarrow
  0+}\frac{||g_1(x)-y|-|g_2(x)-y||}{|y|}=+\infty$. 

A simple way to obtain finite values for the supremum norm is to restrict its
definition to a subset of $\mathcal{Z}$. This corresponds in practice to support assumptions on the data
$(X,Y)$. Hypotheses on $G$ are also needed in general. In this latter case,
one generally assumes $\|G\|_{\infty}<\infty$. In the former case, assumptions
depends on the nature of the loss function.

For instance in the case of the MSE, it is natural to assume that $|Y|$ is
\textbf{upper bounded} by $Y_U$ with probability one. If 
$(x,y)\in \mathbb{R}^d\times [-Y_U,Y_U]$ then
\[
|h_1(x,y)-h_2(x,y)|\leq 2\|G\|_{\infty}(\|G\|_\infty+Y_U),
\]
and therefore the supremum norm is well defined on this subset. 

In the case of the MAPE, a natural hypothesis is that $|Y|$ is \textbf{lower
  bounded} by $Y_L$ (almost surely). If $(x,y)\in\mathbb{R}^d\times
(]-\infty,-Y_L]\cup[Y_L,\infty[)$, then
\[
|h_1(x,y)-h_2(x,y)|\leq 2+2\frac{\|G\|_{\infty}}{Y_L},
\]
and therefore the supremum norm is well defined.

The case of the MAE is slightly different. Indeed when $x$ is fixed, then for
sufficiently large positive values of $y$,
$||g_1(x)-y|-|g_2(x)-y||=|g_1(x)-g_2(x)|$. Similarly, for sufficient large
negative values of $y$, $||g_1(x)-y|-|g_2(x)-y||=|g_1(x)-g_2(x)|$. Thus, the
supremum norm is well defined on $H(G,l_{MAE})$ if
e.g. $\|G\|_{\infty}<\infty$. In addition, we have the following proposition.
\begin{prop}\label{prop:covering:infinite}
  Let $G$ be an arbitrary class of models with $\|G\|<\infty$ and let
  $Y_L>0$. Let $\|.\|^{Y_L}_{\infty}$ denote the supremum norm on $H(G,l_{MAPE})$
  defined by\footnote{Notice that while we make explicit here the dependence
    of the supremum norm on the support on which it is calculated, we will not
    do that in the rest of the paper to avoid cluttered notations. This
    restriction will be clear from the context.  }
\[
\|h\|^{Y_L}_{\infty}=\sup_{x\in\mathbb{R}^d,y\in ]-\infty,-Y_L]\cup[Y_L,\infty[}h(x,y).
\]
Let $\epsilon>0$, then
\[
\mathcal{N}(\epsilon,H(G,l_{MAPE}),\|.\|^{Y_L}_\infty)\leq
  \mathcal{N}(\epsilon Y_L,H(G,l_{MAE}),\|.\|_{\infty}).
\]
\end{prop}
\begin{proof}
Let $\epsilon>0$ and let $h'_1,\ldots,h'_k$ be a minimal $\epsilon Y_L$
cover of $H(G,l_{MAE})$ (thus $k=\mathcal{N}(\epsilon
Y_L,H(G,l_{MAE}),\|.\|_{\infty})$). Let $g_1,\ldots,g_k$ be the functions 
from $G$ associated to $h'_1,\ldots,h'_k$ and let $h_1,\ldots,h_k$ be the
corresponding functions in $H(G,l_{MAPE})$. Then $h_1,\ldots,h_k$ is a
$\epsilon$ cover of $H(G,l_{MAPE})$. 

Indeed let $h$ be an arbitrary element of $H(G,l_{MAPE})$ associated $g$ and
let $h'$ be the corresponding function in $H(G,l_{MAE})$. Then for a given
$j$, $\|h'-h'_j\|_{\infty}\leq\epsilon  Y_L $. We have then
\[
\|h-h_j\|^{Y_L}_{\infty}=\sup_{x\in\mathbb{R}^d,y\in ]-\infty,-Y_L]\cup[Y_L,\infty[}\frac{||g(x)-y|-|g_j(x)-y||}{|y|}.
\]
For all $y\in ]-\infty,-Y_L]\cup[Y_L,\infty[$, $\frac{1}{|y|}\leq
\frac{1}{Y_L}$ and thus
\[
\|h-h_j\|^{Y_L}_{\infty}\leq \sup_{x\in\mathbb{R}^d,y\in ]-\infty,-Y_L]\cup[Y_L,\infty[}\frac{||g(x)-y|-|g_j(x)-y||}{Y_L}.
\]
Then
\begin{align*}
\sup_{x\in\mathbb{R}^d,y\in
  ]-\infty,-Y_L]\cup[Y_L,\infty[}||g(x)-y|-|g_j(x)-y||&\leq   \sup_{x\in\mathbb{R}^d,y\in
  \mathbb{R}}||g(x)-y|-|g_j(x)-y||\\
&\leq\|h'-h'_j\|_{\infty}, \\
&\leq \epsilon Y_L.
\end{align*}
and thus
\[
\|h-h_j\|^{Y_L}_{\infty}\leq \epsilon,
\]
which allows to conclude. 
\end{proof}
This Proposition shows that the covering numbers associated to a class of
functions $G$ under the MAPE are related to the covering numbers of the same
class under the MAE, as long as $Y$ stays away from too small values.

\subsubsection{$L_p$ covering numbers}
$L_p$ covering numbers are based on a data dependent norm. Based on the
training set $D_n$, we define for $p\geq 1$ :
\begin{equation}\label{eq:lp:data:dependent}
\|f_1-f_2\|_{p,D_n}=\left(\frac{1}{n}\sum_{i=1}^n|f_1(Z_i)-f_2(Z_i)|^p\right)^{\frac{1}{p}}.
\end{equation}
We have a simple proposition:
\begin{prop}\label{prop:covering:lp}
Let $G$ be an arbitrary class of models and $D_n$ a data set such that
$\forall i$, $Y_i\neq 0$, then for all $p\geq 1$, 
\[
  \mathcal{N}(\epsilon,H(G,l_{MAPE}),\|.\|_{p,D_n})\leq
  \mathcal{N}(\epsilon\min_{1\leq i\leq N}|Y_i|,H(G,l_{MAE}),\|.\|_{p,D_n}).
\]
\end{prop}
\begin{proof}
The proof is similar to the one of Proposition \ref{prop:covering:infinite}. 
\end{proof}
This Proposition is the adaptation of Proposition \ref{prop:covering:infinite} to $L_p$
covering numbers. 

\subsection{VC-dimension}\label{sec:vc-dimension}
A convenient way to bound covering numbers it to use the Vapnik-Chervonenkis
dimension (VC dimension). We recall first the definition of the shattering
coefficients of a function class. 
\begin{definition}
Let $F$ be a class of functions from $\mathbb{R}^d$ to $\{0,1\}$ and $n$ be a
positive integer. Let $\{z_1,\ldots,z_n\}$ be a set of $n$ points of
$\mathbb{R}^d$. Let 
\[
s(F,\{z_1,\ldots,z_n\})=|\{\theta\in\{0,1\}^n|\exists
f\in F,\ \theta=(f(z_1),\ldots,f(z_n))\}|,
\]
that is the number of different binary vectors of size $n$ that are generated
by functions of $F$ when they are applied to $\{z_1,\ldots,z_n\}$. 

The set $\{z_1,\ldots,z_n\}$ is \textbf{shattered} by $F$ if
$s(F,\{z_1,\ldots,z_n\})=2^n$. 

The $n$-th shatter coefficient of $F$ is 
\[
\mathcal{S}(F,n)=\max_{\{z_1,\ldots,z_n\}\subset \mathbb{R}^d}s(F,\{z_1,\ldots,z_n\}).
\]
\end{definition}
Then the VC-dimension is defined as follows.
\begin{definition}
Let $F$ be a class of functions from $\mathbb{R}^d$ to $\{0,1\}$. The
VC-dimension of $F$ is defined by
\[
VC_{dim}(F)=\sup\{n\in\mathbb{N}^+\mid \mathcal{S}(F,n)=2^n\}.
\]
\end{definition}
Interestingly, replacing the MAE by the MAPE does not increase the VC-dim of the
relevant class of functions.
\begin{prop}\label{prop:vc:dim}
Let $G$ be an arbitrary class of models. We have
\[
VC_{dim}(H^+(G,l_{MAPE}))\leq VC_{dim}(H^+(G,l_{MAE})).
\]  
\end{prop}
\begin{proof}
  Let us consider a set of $k$ points
  shattered by $H^+(G,l_{MAPE})$, $(v_1,\ldots,v_k)$, $v_j=(x_j,y_j,t_j)$. By
  definition, for each binary vector $\theta\in\{0,1\}^k$, there is a function
  $h_{\theta}\in H(G,l_{MAPE})$ such that
  $\forall j,\ \mathbb{I}_{t\leq h_{\theta}(x,y)}(x_j,y_j,t_j)=\theta_j$. Each
  $h_{\theta}$ corresponds to a $g_{\theta}\in G$, with
  $h_{\theta}(x,y)=\frac{|g_{\theta}(x)-y|}{|y|}$. 

  We define a new set of $k$ points, $(w_1,\ldots,w_k)$ as follows. If
  $y_j\neq 0$, then $w_j=(x_j,y_j,|y_j|t_j)$. For those points and for any $g\in G$, 
\[
\mathbb{I}_{t_j\leq \frac{|g(x_j)-y_j|}{|y_j|}}=\mathbb{I}_{|y_j|t_j\leq |g(x_j)-y_j|},
\]
and thus
$\mathbb{I}_{t\leq h_{\theta}(x,y)}(x_j,y_j,t_j)=\mathbb{I}_{t\leq
  h'_{\theta}(x,y)}(x_j,y_j,|y_j|t_j)$ where $h'_\theta(x,y)=|g_{\theta}(x)-y|$.

Let us now consider the case of $y_j=0$. By definition $h_{\theta}(x_j,0)=1$
when $g_\theta(x_j)=0$ and $h_{\theta}(x_j,0)=\infty$ if
$g_\theta(x_j)\neq 0$. As the set of points is shattered $t_j>1$ (or
$h_{\theta}(x_j,0)<t_j$ will never be possible). In addition when $\theta_j=1$
then $g_\theta(x_j)\geq 0$ and when $\theta_j=0$ then $g_\theta(x_j)=0$. Then
let $w_j=(x_j,0,\min_{\theta, \theta_j=1}|g_\theta(x_j)|)$. Notice that
$\min_{\theta, \theta_j=1}|g_\theta(x_j)|>0$ (as there is a finite number of
binary vectors of dimension $k$). For $\theta$ such that $\theta_j=0$, we have
$h'_\theta(x_j,y_j)=|g_{\theta}(x_j)-y_j|=0$ and thus
$h'_\theta(x_j,y_j)<\min_{\theta, \theta_j=1}|g_\theta(x_j)|$, that is
$\mathbb{I}_{t\leq h'_{\theta}(x,y)}(w_j)=0=\theta_j$. For $\theta$ such that
$\theta_j=1$, $h'_\theta(x_j,y_j)=|g_{\theta}(x_j)|$ and thus
$h'_\theta(x_j,y_j)\geq \min_{\theta, \theta_j=1}|g_\theta(x_j)|$. Then
$\mathbb{I}_{t\leq h'_{\theta}(x,y)}(w_j)=1=\theta_j$.

This shows that for each binary vector $\theta\in\{0,1\}^k$, there is a
function $h'_{\theta}\in H(G,l_{MAE})$ such that
$\forall j,\ \mathbb{I}_{t\leq h'_{\theta}(x,y)}(w_j)=\theta_j$. And thus the
$w_j$ are shattered by $H^+(G,l_{MAE})$.

Therefore $VC_{dim}(H^+(G,l_{MAE}))\geq k$. If
$VC_{dim}(H^+(G,l_{MAPE}))<\infty$, then we can take
$k=VC_{dim}(H^+(G,l_{MAPE}))$ to get the conclusion. 

If $VC_{dim}(H^+(G,l_{MAPE}))=\infty$ then $k$ can be chosen arbitrarily large
and therefore $VC_{dim}(H^+(G,l_{MAE}))=\infty$.
\end{proof}
Using theorem 9.4 from \cite{gyorfi_etal_DFTNR2002}, we can bound the $L^p$
covering number with a VC-dim based value.  If
$VC_{dim}(H^+(G,l))\geq 2$, $p\geq 1$, and $0<\epsilon<\frac{\|H(G,l)\|_{\infty}}{4}$,
then
\begin{equation}\label{eq:covering:VC}
\mathcal{N}(\epsilon,H(G,l),\|.\|_{p,D_n})\leq 3\left(\frac{2e\|H(G,l)\|_{\infty}^p}{\epsilon^p}\log \frac{3e\|H(G,l)\|_{\infty}^p}{\epsilon^p}\right)^{VC_{dim}(H^+(G,l))}.
\end{equation}
Therefore, in practice, both the covering numbers and the VC-dimension of MAPE
based classes can be derived from the VC-dimension of MAE based classes. 

\subsection{Examples of Uniform Laws of Large Numbers}\label{sec:exampl-unif-laws}
We show in this section how to apply some of the results obtained above. 

Rephrased with our notations, Lemme 9.1 from \cite{gyorfi_etal_DFTNR2002} is
\begin{lemma}[Lemma 9.1 from \cite{gyorfi_etal_DFTNR2002}]\label{lemma:ulln:supremum}
For all $n$, let $F_n$ be a class of functions from $\mathcal{Z}$ to $[0,B]$
and let $\epsilon>0$. Then
\[
\mathbb{P}\left\{\sup_{f\in
    F_n}\left|\frac{1}{n}\sum_{j=1}^nf(Z_j)-\mathbb{E}(f(Z))\right|\geq\epsilon\right\}\leq
2\mathcal{N}\left(\frac{\epsilon}{3},F_n,\|.\|_{\infty}\right)e^{-\frac{2n\epsilon^2}{9B^2}}.
\]
If in addition 
\[
\sum_{n=1}^\infty\mathcal{N}\left(\frac{\epsilon}{3},F_n,\|.\|_{\infty}\right)<\infty,
\]
for all $\epsilon$, then 
\begin{equation}
\sup_{f\in
    F_n}\left|\frac{1}{n}\sum_{j=1}^nf(Z_j)-\mathbb{E}(f(Z))\right|\rightarrow
  0 \quad (n\rightarrow\infty)\quad a.s.
\end{equation}
\end{lemma}
A direct application of Lemma
\ref{lemma:ulln:supremum} to $H(G,l)$ gives
\[
\mathbb{P}\left\{\sup_{g\in
    G}\left|\widehat{L}_{l}(g,D_n)-L_{l}(g)\right|\geq\epsilon\right\}\leq
2\mathcal{N}\left(\frac{\epsilon}{3},H(G,l),\|.\|_{\infty}\right)e^{-\frac{2n\epsilon^2}{9B^2}},
\]
provided the support of the supremum norm coincides with the support of
$(X,Y)$ and functions in $H(G,l)$ are bounded.

In order to fulfill this latter condition, we have to resort on the same strategy
used to define in a proper way the supremum norm on $H(G,l)$.

As in Section \ref{sec:supr-cover-numb} let $\|G\|_{\infty}<\infty$ and let
$Y_U<\infty$ be such that $|Y|\leq Y_U$ almost surely, then
\begin{equation*}
\forall g\in G, (g(X)-Y)^2\leq \|G\|_{\infty}^2+Y_U^2\quad (a.s.),
\end{equation*}
and
\begin{equation*}
\forall g\in G, |g(X)-Y|\leq \|G\|_{\infty}+Y_U\quad (a.s.).
\end{equation*}
Then if $B\geq \|G\|_{\infty}^2+Y_U^2$ (resp. $B\geq \|G\|_{\infty}+Y_U$),
Lemma \ref{lemma:ulln:supremum} applies to $H(G,l_{MSE})$ (resp. to
$H(G,l_{MAE})$). 

Similar results can be obtained for the MAPE. Indeed let us assume that
$|Y|\geq Y_L>0$ almost surely. Then if $\|G\|_{\infty}$ is finite, 
\begin{equation*}
\forall g\in G, \frac{|g(X)-Y|}{|Y|}\leq 1+ \frac{\|G\|_{\infty}}{Y_L}\quad (a.s.),
\end{equation*}
and therefore for $B\geq 1+ \frac{\|G\|_{\infty}}{Y_L}$, Lemma
\ref{lemma:ulln:supremum} applies to $H(G,l_{MAPE})$. 

This discussion shows that $Y_L$, the lower bound on $|Y|$, plays a very
similar role for the MAPE as the role played by $Y_U$, the upper bound on
$|Y|$, for the MAE and the MSE. A very similar analysis can be made when using
the $L_p$ covering numbers, on the basis of Theorem 9.1 from
\cite{gyorfi_etal_DFTNR2002}. It can also be combined with the results
obtained on the VC-dimension. Rephrased with our notations, Theorem 9.1 from \cite{gyorfi_etal_DFTNR2002} is
\begin{theorem}[Theorem 9.1 from \cite{gyorfi_etal_DFTNR2002}]\label{theorem:ulln:lp}
Let $F$ be a class of functions from $\mathcal{Z}$ to $[0,B]$. Then for
$\epsilon>0$ and $n>0$
\[
\mathbb{P}\left\{\sup_{f\in
    F}\left|\frac{1}{n}\sum_{j=1}^nf(Z_j)-\mathbb{E}(f(Z))\right|\geq\epsilon\right\}\leq
8\mathbb{E}_{D_n}\left\{\mathcal{N}\left(\frac{\epsilon}{8},F,\|.\|_{p,D_n}\right)\right\}e^{-\frac{n\epsilon^2}{128B^2}}.
\]
The expectation of the covering number is taken over the data set $D_n=(Z_i)_{1\leq
  i\leq n}$. 
\end{theorem}
As for Lemma \ref{lemma:ulln:supremum}, we bound $\|H(G,l)\|_\infty$ via assumptions
on $G$ and on $Y$. For instance for the MAE, we have
\begin{multline}
  \label{eq:MAE:ULLN}
\mathbb{P}\left\{\sup_{g\in
    G}\left|\widehat{L}_{MAE}(g,D_n)-L_{MAE}(g)\right|\geq\epsilon\right\}\leq\\
8\mathbb{E}_{D_n}\left\{\mathcal{N}\left(\frac{\epsilon}{8},H(G,l_{MAE}),\|.\|_{p,D_n}\right)\right\}e^{-\frac{n\epsilon^2}{128(\|G\|_{\infty}+Y_U)^2}},
\end{multline}
and for the MAPE
\begin{multline}
  \label{eq:MAPE:ULLN}
\mathbb{P}\left\{\sup_{g\in
    G}\left|\widehat{L}_{MAPE}(g,D_n)-L_{MAPE}(g)\right|\geq\epsilon\right\}\leq\\
8\mathbb{E}_{D_n}\left\{\mathcal{N}\left(\frac{\epsilon}{8},H(G,l_{MAPE}),\|.\|_{p,D_n}\right)\right\}e^{-\frac{n\epsilon^2Y_L^2}{128(1+\|G\|_{\infty})^2}}.
\end{multline}
Equation \eqref{eq:MAPE:ULLN} can be combined with results from Propositions
\ref{prop:covering:lp} or \ref{prop:vc:dim} to allow a comparison between the
MAE and the MAPE. For instance, using the VC-dimension results, the right hand
side of equation \eqref{eq:MAE:ULLN} is bounded above by
\begin{equation}
  \label{eq:MAE:bound}
24  \left(\frac{2e(\|G\|_{\infty}+Y_U)^p}{\epsilon^p}\log \frac{3e(\|G\|_{\infty}+Y_U))^p}{\epsilon^p}\right)^{VC_{dim}(H^+(G,l_{MAE}))}e^{-\frac{n\epsilon^2}{128(\|G\|_{\infty}+Y_U)^2}},
\end{equation}
while the right hand
side of equation \eqref{eq:MAPE:ULLN} is bounded above by
\begin{equation}
  \label{eq:MAPE:bound}
24  \left(\frac{2e(1+\|G\|_{\infty})^p}{Y_L^p\epsilon^p}\log
  \frac{3e(1+\|G\|_{\infty}))^p}{Y_L^p\epsilon^p}\right)^{VC_{dim}(H^+(G,l_{MAPE}))}e^{-\frac{n\epsilon^2Y_L^2}{128(1+\|G\|_{\infty})^2}}.
\end{equation}
In order to obtain almost sure uniform convergence of $\widehat{L}_{l}(g,D_n)$
to $L_{l}(g)$ over $G$, those right hand side quantities must be summable
(this allows one to apply the Borel-Cantelli Lemma). For fixed values of the
VC dimension, of $\|G\|_{\infty}$, $Y_L$ and $Y_U$ this is always the case. If
those quantities are allowed to depend on $n$, then it is obvious, as 
in the case of the supremum covering number, that $Y_U$ and $Y_L$ play
symmetric roles for the MAE and the MAPE. Indeed for the MAE, a fast growth of
$Y_U$ with $n$ might prevent the bounds to be summable. For instance, if $Y_U$
grows faster than $\sqrt{n}$, then $\frac{n}{(\|G\|_{\infty}+Y_U)^2}$ does not
converges to zero and the series is not summable. Similarly, if $Y_L$
converges too quickly to zero, for instance as $\frac{1}{\sqrt{n}}$, then
$\frac{nY_L^2}{(1+\|G\|_{\infty})^2}$ does not converge to zero and the series
is not summable. The following Section goes into more details about those
conditions in the case of the MAPE.

\section{Consistency and the MAPE}\label{sec:consistency-MAPE}
We show in this section that one can build on the ERM principle a strongly
consistent estimator of $L^*_{MAPE}$ with minimal hypothesis on $(X,Y)$ (and
thus almost universal). 
\begin{theorem}\label{theorem:consistency-mape}
Let $Z=(X,Y)$ be a random pair taking values in
$\mathbb{R}^d\times\mathbb{R}$ such that $|Y|\geq Y_L>0$ almost surely ($Y_L$
is a fixed real number). Let $(Z_n)_{n\geq 1}=(X_n,Y_n)_{n\geq 1}$ be a series
of independent copies of $Z$. 

Let $(G_n)_{n\geq 1}$ be a series of classes of measurable functions from
$\mathbb{R}^d$ to $\mathbb{R}$, such that:
\begin{enumerate}
\item $G_{n}\subset G_{n+1}$;
\item $\bigcup_{n\geq 1}G_n$ is dense in
the set of $L^1(\mu)$ functions from $\mathbb{R}^d$ to $\mathbb{R}$ for any
probability measure $\mu$;
\item for all $n$, $V_n=VC_{dim}(H^+(G_n,l_{MAPE}))<\infty$;
\item for all $n$, $\|G_n\|_{\infty}<\infty$.
\end{enumerate}
If in addition
\[
\lim_{n\rightarrow\infty}\frac{V_n\|G_n\|_{\infty}^2\log \|G_n\|_{\infty}}{n}=0,
\]
and there is $\delta>0$ such that 
\[
\lim_{n\rightarrow\infty}\frac{n^{1-\delta}}{\|G_n\|_{\infty}^2}=\infty,
\]
then $L_{MAPE}(\widehat{g}_{l_{MAPE},G_n,D_n})$ converges almost surely to
$L^*_{MAPE}$. 
\end{theorem}
\begin{proof}
We use the standard decomposition between estimation error and approximation
error. More precisely, for $g\in G$, a class of functions, 
\[
L_{MAPE}(g)-L^*_{MAPE}=\underbrace{L_{MAPE}(g)-L^*_{MAPE,G}}_{\text{estimation
    error}}+\underbrace{L^*_{MAPE,G}-L^*_{MAPE}}_{\text{approximation error}}. 
\]
We handle first the approximation error. As pointed out in Section
\ref{sec:setting}, $L_{MAPE}(g)<\infty$ implies that $g\in
L^1(\mathbb{P}_X)$. Therefore we can assume there is a series $(g^*_k)_{k\geq
  1}$ of functions from $L^1(\mathbb{P}_X)$ such that 
\[
L_{MAPE}(g^*_k)\leq L^*_{MAPE}+\frac{1}{k},
\]
by definition of $L^*_{MAPE}$ as an infimum.

Let us consider two models $g_1$ and $g_2$. For arbitrary $x$ and $y$, we have
\begin{align*}
|g_1(x)-y|&\leq |g_1(x)-g_2(x)|+|g_2(x)-y|,\\
|g_2(x)-y|&\leq |g_2(x)-g_1(x)|+|g_1(x)-y|,
\end{align*}
and thus
\begin{align*}
|g_1(x)-y|-|g_2(x)-y|&\leq |g_1(x)-g_2(x)|,\\
|g_2(x)-y|-|g_1(x)-y|&\leq |g_2(x)-g_1(x)|,
\end{align*}
and therefore 
\[
||g_1(x)-y|-|g_2(x)-y||\leq |g_1(x)-g_2(x)|. 
\]
Then 
\[
\left|\mathbb{E}_{X, Y}\left\{\frac{|g_1(X)-Y|}{|Y|}\right\}-\mathbb{E}_{X, Y}\left\{\frac{|g_2(X)-Y|}{|Y|}\right\}\right|\leq \mathbb{E}_{X, Y}\left\{\frac{|g_1(X)-g_2(X)|}{|Y|}\right\}.
\]
As $|Y|\geq Y_L$ almost surely, 
\[
\left|\mathbb{E}_{X, Y}\left\{\frac{|g_1(X)-Y|}{|Y|}\right\}-\mathbb{E}_{X, Y}\left\{\frac{|g_2(X)-Y|}{|Y|}\right\}\right|\leq\frac{1}{Y_L}\mathbb{E}_{X}\left\{|g_1(X)-g_2(X)|\right\},
\]
and thus
\[
|L_{MAPE}(g_1)-L_{MAPE}(g_2)|\leq\frac{1}{Y_L}\mathbb{E}_{X}\left\{|g_1(X)-g_2(X)|\right\}.
\]
As $\bigcup_{n\geq 1}G_n$ is dense in $L^1(\mathbb{P}_X)$ there is a series
$(h^*_k)_{k\geq 1}$ of functions of $\bigcup_{n\geq 1}G_n$ such that
$\mathbb{E}_{X}\left\{|h^*_k(X)-g^*_k(X)|\right\}\leq \frac{Y_L}{k}$. Then
$|L_{MAPE}(h^*_k)-L_{MAPE}(g^*_k)|\leq \frac{1}{Y_L}\frac{Y_L}{k}$ and thus
$L_{MAPE}(h^*_k)\leq L_{MAPE}(g^*_k)+\frac{1}{k}\leq
L^*_{MAPE}+\frac{2}{k}$.
Let $n_k=\min\{n\mid h^*_k\in G_n\}$. By definition,
$L^*_{MAPE,G_{n_k}}\leq L_{MAPE}(h^*_k)$.

Let $\epsilon>0$. Let $k$ be such that $\frac{2}{k}\leq
\epsilon$. Then $L_{MAPE}(g^*_k)\leq L^*_{MAPE}+\epsilon$ and
$L^*_{MAPE,G_{n_k}}\leq L^*_{MAPE}+\epsilon$. Let $n\geq n_k$. As $G_n$ is an
increasing series of sets, $L^*_{MAPE,G_{n}}\leq L^*_{MAPE,G_{n_k}}$ and thus
for all $n\geq n_k$, $L^*_{MAPE,G_{n}}\leq L^*_{MAPE}+\epsilon$. This shows
that $\lim_{n\rightarrow\infty}L^*_{MAPE,G_{n}}=L^*_{MAPE}$. 

The estimation error is handled via the complexity control techniques studied
in the previous Section. Indeed, according to Theorem \ref{theorem:ulln:lp},
we have (for $p=1$)
\[
\mathbb{P}\left\{\sup_{g\in
    G_n}\left|\widehat{L}_{MAPE}(g,D_n)-L_{MAPE}(g)\right|\geq\epsilon\right\}\leq D(n,\epsilon),
\]
with
\[
D(n,\epsilon)=8\mathbb{E}\left\{\mathcal{N}\left(\frac{\epsilon}{8},H(G_n,l_{MAPE}),\|.\|_{1,D_n}\right)\right\}e^{-\frac{n\epsilon^2Y_L^2}{128(1+\|G_n\|_{\infty})^2}}.
\]
Then using equation \eqref{eq:MAPE:bound} 
\[
D(n,\epsilon)\leq  24  \left(\frac{2e(1+\|G_n\|_{\infty})^p}{\epsilon Y_L}\log
  \frac{3e(1+\|G_n\|_{\infty}))}{\epsilon Y_L}\right)^{V_n}e^{-\frac{n\epsilon^2Y_L^2}{128(1+\|G_n\|_{\infty})^2}}.
\]
Using the fact that $\log(x)\leq x$, we have
\[
D(n,\epsilon)\leq  24  \left(\frac{3e(1+\|G_n\|_{\infty})}{\epsilon Y_L}\right)^{2V_n}e^{-\frac{n\epsilon^2Y_L^2}{128(1+\|G_n\|_{\infty})^2}},
\]
and
\begin{align*}
D(n,\epsilon)&\leq 24
                   \exp\left(-\frac{n\epsilon^2Y_L^2}{128(1+\|G_n\|_{\infty})^2}+2V_n\log\frac{3e(1+\|G_n\|_{\infty})}{\epsilon
                   Y_L}\right),\\
&\leq 24\exp\left(-\frac{n}{(1+\|G_n\|_{\infty})^2}\left(\frac{\epsilon^2Y_L^2}{128}-\frac{2V_n(1+\|G_n\|_{\infty})^2\log\frac{3e(1+\|G_n\|_{\infty})}{\epsilon
                   Y_L}}{n}\right)\right).
\end{align*}
As $\lim_{n\rightarrow\infty}\frac{V_n\|G_n\|_{\infty}^2\log
  \|G_n\|_{\infty}}{n}=0$,
\[
\lim_{n\rightarrow\infty}\frac{2V_n(1+\|G_n\|_{\infty})^2\log\frac{3e(1+\|G_n\|_{\infty})}{\epsilon
                   Y_L}}{n}=0.
\]
As $\lim_{n\rightarrow\infty}\frac{n^{1-\delta}}{\|G_n\|_{\infty}^2}=\infty$,
\[
\lim_{n\rightarrow\infty}\frac{n^{1-\delta}}{(1+\|G_n\|_{\infty})^2}=\infty.
\]
Therefore, for $n$ sufficiently large, $D(n,\epsilon)$ is dominated by a term
of the form
\[
\alpha\exp(-\beta n^{\delta}),
\]
with $\alpha>0$ and $\beta>0$ (both depending on $\epsilon$). This allows to
conclude that $\sum_{n\geq 1}D(n,\epsilon)<\infty$. Then the Borel-Cantelli theorem
implies that 
\[
\lim_{n\rightarrow\infty}\sup_{g\in
    G_n}\left|\widehat{L}_{MAPE}(g,D_n)-L_{MAPE}(g)\right|=0 \quad (a.s.).
\]
The final part of the estimation error is handled in a traditional way. 
Let $\epsilon>0$. There is $N$ such that $n\geq N$ implies 
\[
\sup_{g\in
    G_n}\left|\widehat{L}_{MAPE}(g,D_n)-L_{MAPE}(g)\right|\leq \epsilon.
\]
Then
$\widehat{L}_{MAPE}(g,D_n)\leq L_{MAPE}(g)+\epsilon$. By definition
\[
\widehat{L}_{MAPE}(\widehat{g}_{l_{MAPE},G_n,D_n},D_n)\leq\widehat{L}_{MAPE}(g,D_n),
\]
and thus for all $g$,
\[
\widehat{L}_{MAPE}(\widehat{g}_{l_{MAPE},G_n,D_n},D_n)\leq L_{MAPE}(g)+\epsilon.
\]
By taking the infimum on $G_n$, we have therefore
\[
\widehat{L}_{MAPE}(\widehat{g}_{l_{MAPE},G_n,D_n},D_n)\leq L^*_{MAPE,G_n}+\epsilon.
\]
Applying again the hypothesis,
\[
\widehat{L}_{MAPE}(\widehat{g}_{l_{MAPE},G_n,D_n},D_n)\geq L_{MAPE}(\widehat{g}_{l_{MAPE},G_n,D_n})-\epsilon,
\]
and therefore
\[
L_{MAPE}(\widehat{g}_{l_{MAPE},G_n,D_n})\leq L^*_{MAPE,G_n}+2\epsilon.
\]
As a consequence
\[
\lim_{n\rightarrow\infty}|L_{MAPE}(\widehat{g}_{l_{MAPE},G_n,D_n})-L^*_{MAPE,G_n}|=0\quad
(a.s.).
\]
The combination of this result with the approximation result allows us to
conclude.
\end{proof}
Notice that several aspects of this proof are specific to the MAPE. This is
the case of the approximation part which has to take care of $Y$ taking small
values. This is also the case of the estimation part which uses results from
Section \ref{sec:effects-MAPE-compl} that are specific to the MAPE.

\section{MAPE kernel regression}\label{sec:MAPE-kern-regr}
The previous Sections have been dedicated to the analysis of the theoretical
aspects of MAPE regression. In the present Section, we show how to implement MAPE
regression and we compare it to MSE/MAE regression. 

On a practical point of view, building a MAPE regression model consists in
minimizing the empirical estimate of the MAPE 
over a class of models $G_n$, that is to solve
\[
\widehat{g}_{l_{MAPE},G_n,D_n}=\arg\min_{g\in G_n}\frac{1}{n}\sum_{i=1}^n\frac{|g(x_i)-y_i|}{|y_i|},
\]
where the $(x_i,y_i)_{1\leq i\leq n}$ are the realizations of the random
variables $(X_i,Y_i)_{1\leq i\leq n}$.

Optimization wise, this is simply a particular case of \emph{median regression}
(which is in turn a particular case of \emph{quantile regression}). Indeed,
the quotient by $\frac{1}{|y_i|}$ can be seen as a fixed weight and therefore,
any quantile regression implementation that supports instance weights can be
used to find the optimal model. This is for example the case of
\texttt{quantreg} R package \cite{koenker2013quantreg}, among others. Notice that when $G_n$
corresponds to linear models, the optimization problem is a simple
\emph{linear programming} problem that can be solved by e.g. interior point
methods~\cite{boyd2004convex}. 

For some complex models, instance weighting is not immediate. As an example of
MAPE-ing a classical model we show in this section how to turn kernel
quantile regression into kernel MAPE regression. Notice that kernel regression
introduces regularization and thus is not a direct form of ERM. Extending our
theoretical results to the kernel case remains an open question. 

\subsection{From quantile regression to MAPE regression}\label{sec:NPtheory}
\subsubsection{Quantile regression}
Let us assume given a Reproducing Kernel Hilbert Space (RKHS), $\mathcal{H}$, of
functions from $\mathbb{R}^d$ to $\mathbb{R}$ (notice that $\mathbb{R}^d$
could be replaced by an arbitrary space $\mathcal{X}$). The associated kernel
function is denoted $k$ and the mapping between $\mathbb{R}^d$ and
$\mathcal{H}$, $\phi$. As always, we have $k(x, x') = \langle \phi(x),
\phi(x')\rangle$. 

The standard way of building regression models based on a RKHS consists in
optimizing a regularized version of an empirical loss, i.e., in solving an
optimization problem of the form
\begin{equation}
  \label{eq:rkhs:general}
  \min_{f \in \mathcal{H},b\in\mathbb{R}} \sum_{i = 1}^n l(f(x_i)+b,y_i)) + \frac{\lambda}{2} \| f \|^2_{\mathcal{H}}.
\end{equation}
Notice that the reproducing property of $\mathcal{H}$ implies that there is
$w\in \mathcal{H}$ such that $f(x)=\langle w,\phi(x)\rangle$. 

In particular, quantile regression can be kernelized via an appropriate choice
for $l$. Indeed, let $\tau \in [0;1]$
and let $\rho_\tau$ be the \textit{check-function}, introduced in \cite{koenker1978regression}: 
\[ \rho_\tau(\xi) = \left\{ 
\begin{matrix}
\tau \xi & \mbox{ if } \xi \geq 0 \\
(\tau - 1) \xi & \mbox{ otherwise }\\
\end{matrix}
\right.\]
The check-function is also called the \emph{pinball loss}. Then, the kernel
quantile optimization problem, treated in \cite{takeuchi2006nonparametric,
  li2007quantile}, is defined by:
\begin{equation}
\label{eq:NPmae}
\min_{f \in \mathcal{H},b\in\mathbb{R}} \sum_{i = 1}^n \rho_\tau(y_i - f(x_i)-b) + \frac{\lambda}{2} \| f \|^2_{\mathcal{H}},
\end{equation}
where $\lambda > 0$ handles the trade-off between the data fitting term and the
regularization term. The value of $\tau$ gives the quantile that the model $f$
is optimized for: for instance $\tau=\frac{1}{2}$ corresponds to the median. 

\subsection{MAPE primal problem}
To consider the case of the MAPE, one can change the equation \eqref{eq:NPmae} to \eqref{eq:NPMAPE}:
\begin{equation}
\label{eq:NPMAPE}
\min_{f \in \mathcal{H}} \sum_{i = 1}^n \frac{\rho_\tau(y_i - f(x_i)-b)}{|y_i|} + \frac{\lambda}{2} \| f \|^2_{\mathcal{H}}.
\end{equation}
Notice that for the sake of generality, we do not specify the value of
$\tau$ in this derivation: thus equation \eqref{eq:NPMAPE} can be seen as a
form of ``relative quantile''. However, in the simulation study in Section
\ref{sec:simulation-study}, we limit ourselves to the standard MAPE, that is
to $\tau=\frac{1}{2}$. The practical relevance of the ``relative quantile''
remains to be assessed.  

Using the standard way of handling absolute values and using  $f(x) = \langle
\phi(x), w \rangle$, we can rewrite the regularization problem
\eqref{eq:NPMAPE} as a (primal) optimization problem:
\begin{eqnarray}
\min_{w, b, \xi, \xi^\star} & C \sum_{i = 1}^n \frac{\tau \xi_i + (1 - \tau) \xi_i^\star}{|y_i|}  + \frac{1}{2} \| w \|^2,  \label{eq:primal}\\
\mbox{subject to} 	& y_i - \langle \phi(x_i) , w \rangle - b \leq |y_i| \xi_i, \forall i, 		 \nonumber\\
			& \langle \phi(x_i) , w \rangle + b - y_i  \leq |y_i|\xi_i^\star, \forall i, 	\nonumber\\
			& \xi_i \geq 0, \forall i,									 \nonumber\\
			& \xi_i^\star \geq 0, \forall i, 								\nonumber 
\end{eqnarray}
where $C = \frac{1}{n\lambda}$.

\subsubsection{MAPE dual problem}
Let us denote $\theta=(w,b,\xi, \xi^\star)$ the vector regrouping all the variables
of the primal problem. We denote in addition:
\begin{eqnarray*}
h(\theta) & = & C \sum_{i = 1}^n \frac{\tau \xi_i + (1 - \tau) \xi_i^\star}{|y_i|}  + \frac{1}{2} \| w \|^2, \\
\forall i,\ g_{i,1}(\theta) & = & y_i - \langle \phi(x_i) , w \rangle - b - |y_i|\xi_i,\\
\forall i,\ g_{i,2}(\theta) & = & \langle \phi(x_i) , w \rangle + b - y_i  -   |y_i|\xi_i^\star,  \\
\forall i,\ g_{i,3}(\theta) & = & - \xi_i, \\
\forall i,\ g_{i,4}(\theta) & = & - \xi_i^\star.\\
\end{eqnarray*}
Then the Wolfe Dual of problem \eqref{eq:primal} is given by:
\begin{eqnarray}
\max_{\theta, u} & h(\theta) + \sum_{i = 1}^n \left(u_{i,1}g_{i,1}(\theta) + u_{i,2}g_{i,2}(\theta) + u_{i,3}g_{i,3}(\theta) + u_{i,4}g_{i,4}(\theta)\right), \label{eq:Wolfe0}\\
\mbox{s. t.} 	& \nabla h(\theta) + \sum_{i = 1}^n \left(u_{i,1}\nabla g_{i,1}(\theta) + u_{i,2}\nabla g_{i,2}(\theta) + u_{i,3}\nabla g_{i,3}(\theta) + u_{i,4}\nabla g_{i,4}(\theta)\right) = 0, \nonumber\\
			& u_{i,1}, u_{i,2}, u_{i,3}, u_{i,4} \geq 0, \forall i, \nonumber
\end{eqnarray}
where the $u_{i,k}$ are the Lagrange multipliers. Some algebraic manipulations
show that problem \eqref{eq:Wolfe0} is equivalent to problem \eqref{eq:Wolfe}:
\begin{eqnarray}
\label{eq:Wolfe}
\max_{\theta, u} & h(\theta) + \sum_{i = 1}^n \left(u_{i,1}g_{i,1}(\theta) + u_{i,2}g_{i,2}(\theta) + u_{i,3}g_{i,3}(\theta) + u_{i,4}g_{i,4}(\theta)\right), \\
\mbox{s. t.} 	&w + \sum_{i = 1}^n(  u_{i,2}-u_{i,1}) \phi(x_i) = 0, 	\label{constraint:c1}	\\
			& \sum_{i = 1}^n (u_{i,2} - u_{i,1}) = 0, 					\label{constraint:c2}	\\
			& \forall i,\ \frac{C\tau}{|y_i|} - |y_i| u_{i,1} - u_{i,3}= 0, 				\label{constraint:c3}	\\
			& \forall i,\ \frac{C(1 - \tau)}{|y_i|} - |y_i| u_{i,2} - u_{i,4}= 0, 			\label{constraint:c4}	\\
			& \forall i,\ u_{i,1}, u_{i,2}, u_{i,3}, u_{i,4} \geq 0. \label{constraint:c5}
\end{eqnarray}
We can simplify the problem by introducing a new parametrisation via the
variables $\alpha_i = u_{i,1} - u_{i,2}$. Then the value of $w$ is obtained
from constraint \eqref{constraint:c1} as
$w=\sum_{i=1}^n\alpha_i\phi(x_i)$. Constraints \eqref{constraint:c2} can be
rewritten into $1^T\alpha = 0$. Taking those equations into account, the
objective function becomes
\begin{multline*}
  h(\theta) + \sum_{i = 1}^n \left(u_{i,1}g_{i,1}(\theta) +
    u_{i,2}g_{i,2}(\theta) + u_{i,3}g_{i,3}(\theta) +
    u_{i,4}g_{i,4}(\theta)\right)\\
=h(\theta)+\sum_{i=1}^n\alpha_iy_i-\|w\|^2-\sum_{i=1}^n\xi_i(u_{i,1}|y_i|+u_{i_3})-\sum_{i=1}^n\xi^*_i(u_{i,2}|y_i|+u_{i_4}).
\end{multline*}
Using constraints \eqref{constraint:c3} and \eqref{constraint:c4}, the last
two terms simplify as follows:
\[
\sum_{i=1}^n\xi_i(u_{i,1}|y_i|+u_{i_3})+\sum_{i=1}^n\xi^*_i(u_{i,2}|y_i|+u_{i_4})=C \sum_{i = 1}^n \frac{\tau \xi_i + (1 - \tau) \xi_i^\star}{|y_i|},
\]
and thus the objective function is given by
\begin{multline*}
  h(\theta) + \sum_{i = 1}^n \left(u_{i,1}g_{i,1}(\theta) +
    u_{i,2}g_{i,2}(\theta) + u_{i,3}g_{i,3}(\theta) +
    u_{i,4}g_{i,4}(\theta)\right)\\
=\sum_{i=1}^n\alpha_iy_i-\frac{1}{2}\|w\|^2=\alpha^Ty-\frac{1}{2}\alpha^T K \alpha,
\end{multline*}
where $K_{ij}=k(x_i,x_j)$ is the kernel matrix. This shows that the objective
function can be rewritten so as to depend only on the new variables
$\alpha_i$. The last step of the analysis consists in showing that a similar
property holds for the constraints. The cases of constraints
\eqref{constraint:c1} and \eqref{constraint:c2} have already been
handled. 

Notice that given an arbitrary $\alpha_i$, there is always $u_{i,1}\geq 0$ and
$u_{i,2}\geq 0$ such that $\alpha_i = u_{i,1} - u_{i,2}$. However,
constraints \eqref{constraint:c3} and \eqref{constraint:c4} combined with
$u_{i,3}\geq 0$ and $u_{i,4}\geq 0$ show that $u_{i,1}$ and $u_{i,2}$ (and
thus $\alpha_i$) cannot
be arbitrary, as we need $\frac{C\tau}{|y_i|} - |y_i| u_{i,1}\geq 0$ and
$\frac{C(1 - \tau)}{|y_i|} - |y_i| u_{i,2}\geq 0$. As $u_{i,2}\geq 0$,
$\alpha_i\leq u_{i,1}$ and thus $\alpha_i\leq \frac{C\tau}{|y_i|^2}$. As $u_{i,1}\geq 0$,
$-\alpha_i\leq u_{i,2}$ and thus $\alpha_i\geq
\frac{C(1-\tau)}{|y_i|^2}$. Conversely, it is easy to see that if $\alpha_i$
satisfies the constraints $\frac{C(\tau - 1)}{|y_i|^2} \leq \alpha_i  \leq
\frac{C \tau}{|y_i|^2}$, then there is $u_{i,k}$ for $k=1,\ldots,4$ 
 such that $\alpha_i = u_{i,1} - u_{i,2}$ and such that the
constraints \eqref{constraint:c3}, \eqref{constraint:c4} and
\eqref{constraint:c5} are satisfied (take $u_{i,1}=\max(0,\alpha_i)$ and
$u_{i,2}=\max(0,-\alpha_i)$).

Then problem \eqref{eq:Wolfe} is finally equivalent to 
\begin{eqnarray}
\label{eq:DualMAPE}
\max_{\alpha} & \alpha^T y - \frac{1}{2} \alpha^T K \alpha  \label{eq:Wolfe1} \\
\mbox{s.c.} 	& 1^T \alpha = 0 \nonumber \\
			& \forall i,\ \frac{C(\tau - 1)}{|y_i|^2} \leq \alpha_i  \leq \frac{C \tau}{|y_i|^2}. \nonumber
\end{eqnarray}

\subsubsection{Comparaison to the quantile regression}

In the case of quantile regression, \cite{takeuchi2006nonparametric} shows that the dual problem is equivalent to
\begin{eqnarray*}
\max_{\alpha} & \alpha^T y - \frac{1}{2}\alpha^T K \alpha   \\
\mbox{s.c.} 	& 1^T \alpha = 0 \\
			& \forall i,\ C(\tau - 1) \leq \alpha_i  \leq C \tau.
\end{eqnarray*}

In comparison to problem \eqref{eq:DualMAPE}, one can remark that the
modification of the loss function (from the absolute error to the absolute
percentage error) in the primal optimization problem is equivalent to changing
the set of optimization in the dual optimization problem. More precisely, it
is equivalent to reducing (resp. increasing) the ``size'' of the optimization set
of $\alpha_i$ if $y_i > 1$ (resp. $y_i < 1$). 

Thus, the smaller is $y_i$, the larger is the optimization set of
$\alpha_i$. This permits to ensure a better fit on small values of $y_i$
(i.e. where the absolute percentage error is potentially bigger). Moreover, by
choosing a very large value of $C$ (or $C \to \infty$), one can ensure the
same optimal value of each $\alpha_i$ in MAE and MAPE dual problems. This
surprising fact can be explained by noticing that a very large value of $C$
corresponds to a very small value of $\lambda$ (or $\lambda \to 0$). When
$\lambda$ goes to zero, the regularization in equations \eqref{eq:NPmae} and
\eqref{eq:NPMAPE} vanishes, which leads to potential overfitting. When this
overfitting appears, $f(x_i)\simeq y_i$ regardless of the loss function and
thus the different loss functions are equivalent.

\subsection{A simulation study}\label{sec:simulation-study}

\subsubsection{Generation of observations} 
In this section, we illustrate the efficiency of the kernel MAPE regression
described in section~\ref{sec:NPtheory}  on simulated data, and we compare the
results to the ones obtained by kernel median regression. Experiments have been realized using a Gaussian kernel.

As in \cite{takeuchi2006nonparametric}, we have simulated data according to the sinus cardinal function, defined by
\[ sinc(x) = \frac{sin(2\pi x)}{2\pi x} \]
However, to illustrate the variation of the prediction according the proximity
to zero, we add a parameter $a$ and we define the translated sinus cardinal
function by: 
\[ sinc(x,a) = a + \frac{sin(2\pi x)}{2\pi x} \]
For experiments, we have generated 1000 points to constitute a training set, and
1000 other points to constitute a test set. As in
\cite{takeuchi2006nonparametric}, the generation process is the following:
\[ Y = sinc(X,a) + \epsilon(X) \]
with $X \sim \mathcal{U([-1;1])}$ and
$\epsilon(X) \sim \mathcal{N}\left(0, \left(0.1\cdot\exp(1-X)\right)^2\right)$

To compare the results between the median estimation and the MAPE estimation,
we have computed $\widehat{f}_{MAPE,a}$ and $\widehat{f}_{MAE,a}$ for several
values of $a$. The value of the regularization parameter $C$ is chosen via a 5-fold cross-validation.

\subsubsection{Results}

\begin{table}[ht]
\centering
\begin{tabular}{rrrrr}
\toprule
a 	& $MAPE(y, \widehat{f}_{MAE, a})$ 	& $MAPE(y, \widehat{f}_{MAPE, a})$ 	& $C_{MAE}$ 	& $C_{MAPE}$ \\ 
 	&  (in \%) 							& 	(in \%)						&	 		& 			\\ 
  \midrule
0.00 & 128.62 & 94.09 & 0.01 & 0.10 \\ 
  0.10 & 187.78 & 100.10 & 0.05 & 0.01 \\ 
  0.50 & 72.27 & 57.47 & 5.00 & 10.00 \\ 
  1.00 & 51.39 & 39.53 & 10000.00 & 1.00 \\ 
  2.50 & 10.58 & 10.98 & 5.00 & 1.00 \\ 
  5.00 & 4.80 & 4.89 & 5.00 & 10.00 \\ 
  10.00 & 2.39 & 2.40 & 5.00 & 100.00 \\ 
  25.00 & 0.96 & 0.96 & 5.00 & 100000.00 \\ 
  50.00 & 0.48 & 0.48 & 5.00 & 1000.00 \\ 
  100.00 & 0.24 & 0.24 & 5.00 & 10000.00 \\ 
   \bottomrule
\end{tabular}
\caption{Summary of the experimental results: for each value of the
  translation parameter $a$, the table gives the MAPE of
  $\widehat{f}_{MAPE,a}$ and $\widehat{f}_{MAE,a}$ estimated on the test
  set. The table also reports the value of the regularization parameter $C$
  for both loss function.}
\label{tab:res}
\end{table}

Results of experiments are described in the table~\ref{tab:res}. As expected,
in most of the cases, the MAPE of $\widehat{f}_{MAPE, a}$ is lower than the
one of $\widehat{f}_{MAE, a}$. This is 
especially the case when values of $y$ are close to zero. 

\subsubsection{Graphical illustration}
Some graphical representations of $\widehat{f}_{MAPE, a}$ and
$\widehat{f}_{MAE, a}$ are given on Figure~\ref{fig:res}. This Figure
illustrates several interesting points: 
\begin{itemize}
  \item When, for a given $x$, $y$ may take both negative and positive values,
    $\widehat{f}_{MAPE, a}(x)$ (red curve) is very close or equal to 0 to
    ensure a 100\% error whereas $\widehat{f}_{MAE, a}(x)$ (blue curve)
    is closer to the conditional median, which leads to a strongly higher
    error (in MAPE terms). 
  \item Up to translation, $\widehat{f}_{MAE, a}$ looks roughly the same for
    each $a$, whereas the shape of $\widehat{f}_{MAPE, a}(x)$ is strongly
    modified with $a$. This is because the absolute error (optimization
    criteria for the blue curve) remains the same if both the observed value
    $Y$ and its predicted value are translated by the same value, whereas the
    MAPE changes.
  \item Red curves are closer to 0 than blue curves. One can actually show
    that, regarding to the MAPE, the optimal estimator (red) of a random
    variable $Y$ is indeed below the median (blue). 
  \item The red curve seems to converge toward the blue one for high values of
    $a$. 
\end{itemize}

\begin{figure}
\centering
\includegraphics[width =\linewidth]{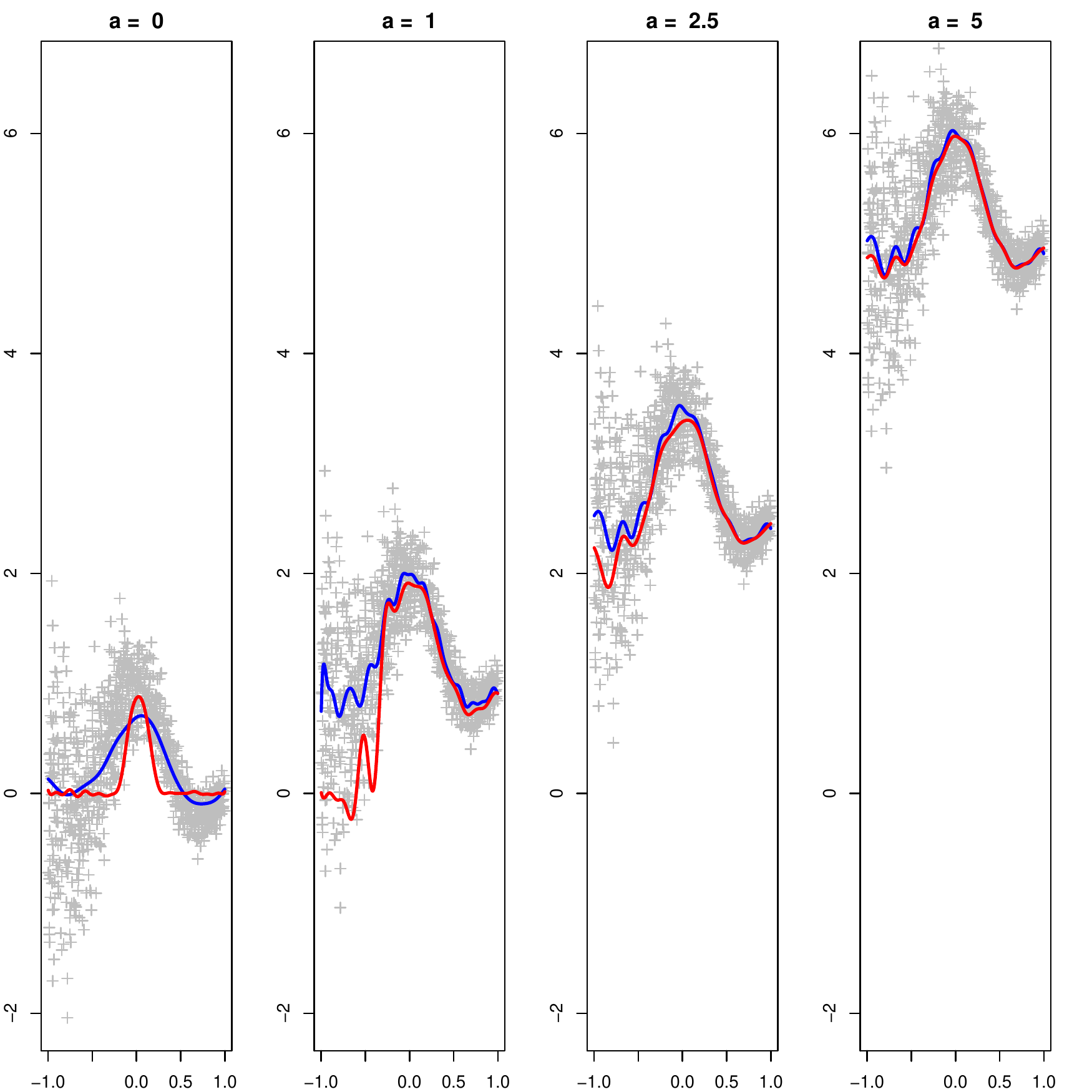}
\caption{Representation of estimation: $\widehat{f}_{MAE, a}$ in blue and $\widehat{f}_{MAPE, a}$ in red.}
\label{fig:res}
\end{figure}

\section{Conclusion}
We have shown that learning under the Mean Absolute Percentage Error is
feasible both on a practical point of view and on a theoretical one. More
particularly, we have shown the existence of an optimal model regarding to the
MAPE and the consistency of the Empirical Risk Minimization. Experimental
results on simulated data illustrate the efficiency of our approach to
minimize the MAPE through kernel regressions, what also ensures its efficiency
in application contexts where this error measure is adapted (in general when
the target variable is positive by design and remains quite far away from
zero, e.g. in price prediction for expensive goods). Two open theoretical
questions can be formulated from this work. A first question is  whether the
lower bound hypothesis on $|Y|$ can be lifted: in the case of MSE based
regression, the upper bound hypothesis on $|Y|$ is handled via some clipping
strategy (see e.g. Theorem 10.3 in \cite{gyorfi_etal_DFTNR2002}). This cannot
be adapted immediately to the MAPE because of the importance of the lower
bound on $|Y|$ in the approximation part of Theorem
\ref{theorem:consistency-mape}. A second question is whether the case of
empirical regularized risk minimization can be shown to be consistent in the
case of the MAPE.

\section*{Acknowledgment}
The authors thank the anonymous reviewers for their
valuable comments that helped improving this paper. 

\section*{References}
\bibliography{mape}

\end{document}